\documentclass[runningheads]{llncs}
\usepackage{times}

\usepackage{soul}
\usepackage{url}
\usepackage[hidelinks]{hyperref}
\usepackage[utf8]{inputenc}
\usepackage[small]{caption}
\usepackage{graphicx}
\usepackage{amsmath}
\usepackage{amsfonts}
\usepackage{amssymb}
\usepackage{booktabs}
\usepackage{epsfig}
\usepackage{qtree}
\usepackage[]{algorithm2e}
\usepackage{algorithmic}
\usepackage{float}
\usepackage{times}
\usepackage{latexsym}
\usepackage{verbatim}
\usepackage{enumitem}

\newcommand{\ifaf}{if and only if}
\newcommand{\suchthat}{s.t.}
\newcommand{\pos}{\ensuremath{\mathtt{POS}}}
\newcommand{\npos}{\ensuremath{\mathtt{NEG}}}

\newtheorem{prenot}{Notation}[section]

\begin{document}
\title{Evaluating the Correctness of Explainable AI Algorithms for Classification}
%
%
\author{Orcun Yalcin, Xiuyi Fan, Siyuan Liu}
%
%
\institute{Swansea University, The UK, SA2 8PP} 
%
\maketitle              
\begin{abstract}
Explainable AI has attracted much research attention in
recent years with feature attribution algorithms, which compute ``feature
importance'' in predictions, becoming increasingly popular. However,
there is little analysis of the validity of these algorithms as there is
no ``ground truth'' in the existing datasets to validate their correctness. In this work,
we develop a method to quantitatively evaluate the correctness of
XAI algorithms by creating datasets with known explanation ground
truth. To this end, we focus on the binary classification
problems. String datasets are constructed using formal language derived from a grammar. A string is {\em positive} if and only if a certain property is fulfilled. Symbols serving as explanation ground truth in a positive string are part
of an explanation if and only if they contributes to fulfilling the property. Two popular feature attribution explainers,
Local Interpretable Model-agnostic Explanations (LIME) and
SHapley Additive exPlanations (SHAP), are used in our experiments.We show that: (1) classification accuracy is
positively correlated with explanation accuracy; (2) SHAP provides
more accurate explanations than LIME; (3) explanation accuracy is
negatively correlated with dataset complexity.

\end{abstract}
\section{Introduction}
\label{sec:intro}

Explainable AI (XAI) is a fast growing research area in AI that aims
to provide insight into processes that AI uses to
conclude \cite{Adadi18}. The goal of enabling explainability in AI
systems is ``to ensure algorithmic predictions and any input data
triggering those predictions can be explained to
non-experts'' \cite{Carvalho19}.  XAI can create practical machine
learning methods that produce more human-understandable models while
preserving a high accuracy of predictions.

Two main categories of approaches have been proposed in the literature to
address the need for explainability in AI systems: (1) intrinsically interpretable methods
\cite{Rudin19}, in which prediction and explanation are both
produced by the same underlying mechanism, and (2) model-agnostic methods
\cite{molnar2019}, in which explanations are treated as a post hoc
exercise and are separated from the prediction model. In the case for
methods (1), while many intrinsically interpretable models, such as
short decision trees, linear regression, Naive Bayes, k-nearest
neighbours and decision rules \cite{Yang17} are easy
to understand, they can be weak for prediction and suffer from
performance loss in complex tasks. On the contrary, model agnostic
approaches such as local surrogate \cite{Ribeiro16}, global surrogate
\cite{Alonso18}, feature importance \cite{Fisher18}, and
symbolic Bayesian network transformation \cite{Shih18} will separate explanation from prediction and produce a comparatively better prediction result.

Among the model agnostic
methods, giving users explanations in the form of feature importance has been viewed as an effective approach in XAI -- as each feature makes some contribution to the prediction, by knowing
the ``weights'' of features, one can better understand how a prediction is made. In addition to the Local Interpretable
Model-Agnostic Explanations (LIME) introduced in \cite{Ribeiro0G16}, competing feature importance approaches
including the SHapley Additive exPlanations (SHAP) \cite{Lundberg17} have been developed.
Although they have attracted much research attention, there is little work on evaluating the correctness of the explanations produced by the feature importance explainers. After all, what explainers do is to output numbers associated to features. How can we know the feature importance asserted by such explainers is correct explanations? Taking the well known Mushroom dataset \cite{Dua17} as an example, there are 8124
mushroom instances in this dataset where each instance is either edible or poisonous. With a standard random forest classifier, one achieves classification AUC
0.98 on a 80/20 training/testing split. However, upon questioning top features
for predicting the 3916 poisonous instances, LIME and SHAP differ on 1527 of
them when the highest ranked feature is considered; and 2744 of them when
the top two highest ranked features are considered. The question is: Is LIME or SHAP is correct, or are both correct despite they reporting drastically different
top features?

With the few exceptions discussed in Section~\ref{sec:relatedWork},
literatures on evaluating correctness are hard to
find as only 5\% of the researchers that have studied AI
interpretability or explainability have focused on assessing
them \cite{Adadi18}. Existing evaluations, see e.g. \cite{Mohseni20}
for a recent overview, are predominantly based on human
inputs. However, as discussed by \cite{Das20,Arrieta20,Ignatiev20},
any evaluation metric that is completely based on human
inputs is fallible due to human bias and subjectivity. Even though
some other metrics for XAI methods such as ``explanation
goodness'', ``explanation satisfaction scale'', ``explanation
usefulness'', and ``user trust'' are discussed
in \cite{Hoffman18,Mohseni20}, quantitative analysis on explanation
correctness without human input is needed.

Various reasons lead to the lack of a quantitative study on correctness of
XAI methods. The chief reason is the lack of {\em explanation truth} in
datasets. For example, in the Mushroom
dataset, it is impossible to know what the most important features for a mushroom to be
poisonous are from the
dataset itself. In other words, there is no ground truth to validate
the correctness of XAI methods. Therefore, in this work, we give an approach to generate datasets for binary classification problems with explanation ground truth. In this way, we can compare feature importance XAI methods
against the designed explanation ground truth while performing prediction tasks
as normal. We do not rely on human inputs in our evaluation as the explanation ground truth are coded into our datasets. Although our dataset generation
method may not capture all real world classification tasks, it can be viewed as a
systematic and sound benchmark for comparing feature importance XAI methods.



%

More specifically, we formally define correct explanations for binary classification
and introduce an evaluation metric, {\em $k$-accuracy}, in
Section~\ref{sec:exp_classification}. We present an algorithm for
generating datasets with correct explanation ground truth using formal
grammar in Section~\ref{sec:dataset}. We then introduce {\em
grammatical complexity (G-complexity)}, modelled as the Kolmogorov
Complexity \cite{Li19}, as a controllable complexity
measure for dataset generation in Section~\ref{sec:complexity} and
show performances of classification and explanation are negatively
correlate to it. We experiment with SHAP and LIME throughout in Section ~\ref{sec:evaluation}.

Contributions of this paper are: (1) a method for generating datasets
that allow qualitative study of explanation correctness; (2) the
discovery of a positive relation between classification accuracy and
explanation performance; (3) quantitative evaluation of SHAP and LIME
with the proposed method; and (4) the discovery of a negative relation
between SHAP performance and dataset complexity.

\section{Explanation for Binary Classification}
\label{sec:exp_classification}

In this work, we focus on the problem of binary classification on categorical data. To construct datasets with explanation ground truth, we need a way to generate data instances that can be labelled in two ways, positive ($\pos$) and negative
($npos$). For each instance, we need to be able to clearly specify features that are
``responsible'' for the labelling. Not to trivialise the process, the classification
labelling cannot be solely determined by some fixed features in the dataset although such features would explain the classification. For example, \emph{``when feature 2 is above 50, then label the instance $\pos$,''} would be considered an over simplification for real world classification tasks. We must generate data instances such that the classification labelling is determined
by a subset of feature values; yet, the position of explanation subset cannot be
fixed throughout the dataset.

To achieve this, we use formal grammars and their corresponding languages. The language and grammar notions we will use are as follows.
\subsection{Language and Grammar Notions}
\begin{itemize}[noitemsep]
\item
  \emph{Alphabet} $\Sigma$ is a finite set of symbols.

\item
  \emph{String} is a finite sequence of symbols from $\Sigma$.

\item
  \emph{Language} $L$ is a subset of $\Sigma^{*}$.
\end{itemize}
A Grammar $G$ is a set of production rules and it describes the
strings in the language. Formally, given alphabet $\Sigma$, a grammar
$G$ is a tuple $\langle N, T, P, S \rangle$, where:
\begin{itemize}[noitemsep]
    \item
      $N$ is a set of nonterminal symbols, $N \cap \Sigma = \{\}$;

    \item
      $T = \Sigma$ is a set of terminal symbols;

    \item
      $P$ is a set of production rules from $N$ to
      $(N \cup T)^*$; and

    \item
      $S \in N$ the start symbol.
\end{itemize}
The parse trees for $G$ are tress with the following conditions:
\begin{enumerate}[noitemsep]
\item
Each interior node is labelled by a symbol in $N$.

\item
Each leaf is labelled by either a symbol in $N$, a string of terminal
symbols in $T$, or $\epsilon$.

\item
If an interior node is labelled as $A$, and its children are labelled as
$X_1, X_2, \ldots, X_k$ respectively, from the left, then
$A \rightarrow X_1, X_2, \ldots, X_k$ is a production rule in $P$.

\end{enumerate}

A grammar is ambiguous if there exists a string in the language of the
grammar \suchthat{} the string can have more than one parse tree.

\subsection{Correct Explanation}
We formulate the binary clarification problem on categorical data as the following {\em string classification}. Given a set of strings $S$, \suchthat{} each $s \in S$ is of
the same length $k > 0$; there exists a labelling function $h$
which maps each string to a class $c \in C$, with $C$ being all
possible classes for $S$. The classification task is to identify a
classifier $g: S \mapsto C$ \suchthat{} $g(s) = h(s)$. In this process,
each symbol in $s$ is a feature of $s$. For instance, the string 0011
has four features, $f_1, f_2, f_3, f_4$, with
\begin{center}
$f_1$ = 0, $f_2$ = 0, $f_3$ = 1, and $f_4$ = 1;
\end{center}
and the string 1100 has four features with
\begin{center}
$f_1$ = 1, $f_2$ = 1, $f_3$ = 0, and $f_4$ = 0.
\end{center}
Many real binary classification problems are the instances of the string classification. For example, the Mushroom dataset mentioned in Section 1 is a
dataset with strings of length 22 on the alphabet
$\{a,b,c,d,e,f,g,h,$ $k,l,m,n,o,p,r,s,t,u,v,w,x,y,z,?\}$.\footnote{See
  \cite{Dua17} for the meaning of feature symbols.} A instance of the mushroom dampest is shown in Table \ref{table:1}.
  
\vspace{-15pt}
\begin{table}[!ht]
\caption{Explanation given by SHAP and LIME for a correctly classified
data instance from the Mushroom dataset.\label{table:1}}
\begin{center}
\begin{tabular}{|p{3.8cm}|p{3.8cm}|}
 \hline
 \multicolumn{2}{|l|}{Poisonous mushroom: {\bf xwnytffcbtbsswwpwopksu}} \\
 \hline
 \multicolumn{2}{|p{8.1cm}|}
 {Cap shape: conve{\bf x}, cap colour: {\bf w}hite, gill colour:
 brow{\bf n}, cap surface: scal{\bf y},  bruises: {\bf t}=bruises, odor:
 {\bf f}oul, gill attachment: {\bf f}ree, gill spacing: {\bf c}lose,
 gill size: {\bf b}road, stalk shape: {\bf t}apering, stalk root: {\bf
 b}ulbous, stalk-surface-above-ring: {\bf s}mooth,
 stalk-surface-below-ring: {\bf s}mooth, stalk-color-above-ring :
 {\bf w}hite, stalk-color-below-ring: {\bf w}hite, veil-type: {\bf
 p}artial, veil-color: {\bf w}hite, ring-number: {\bf o}ne, ring-type:
 {\bf p}endant, spore-print-color: blac{\bf k},  population:
 {\bf s}cattered, habitat: {\bf u}rban.} \\
 \hline
\end{tabular}
\vspace{-15pt}
\end{center}
\end{table}
  
Each string in
this dataset is labelled in one of the two classes: poisonous or
edible. Since each instance can belong to only one of the two classes,
classification on this dataset is binary.


Before we define \emph{correct explanation}, We first give the definition of {\em substring}. 
\begin{definition}
Given a string $s = a_1 \ldots a_n$ over alphabet $\Sigma$, where
$\cdot \not\in \Sigma$, $s$ represents a data point $x$ with features
$f_1 \ldots f_n$, \suchthat{} $x \!=\! \{f_1 \!=\! a_1, \ldots,
f_n \!=\! a_n\}$. A {\em substring} $s'$ of $s$ is a string $s' \!=\!
a_1' \ldots a_n'$ over alphabet $\Sigma \cup \{\cdot\}$ with
$a_i' \in \{a_i, \cdot\}$, $i = 1 \ldots n$. $s'$ represents the set
$\{f_i \!=\! a_i \in x| a_i' \neq \cdot\}$.
\end{definition}
For instance, the substring $\cdot01\cdot$ of the string 0011 denotes
the set of two feature-values $\{f_2 = 0, f_3 = 1\}$.

We define our notions of {\em explanation} (for a $\pos$
classification) and {\em correct explanation} as follows.
\begin{definition}
  \label{dfn:explanation}

  Given a string $x \in S$ \suchthat{} $h(x) = \pos$, an {\em
    explanation} $e_x$ for the label $h(x)=\pos$ is a substring of
  $x$. We also say $e_x$ is an explanation for $x$ when there is no
    ambiguity.

  An explanation $e_x$ is {\em correct} (for $x$ being $\pos$)
  \ifaf{} for any string $x' \in S$, if $e_x$ is a substring of $x'$,
  then $h(x)=h(x')=\pos$.
  If $e_x$ is a correct explanation (for some label $h(x)$) and
  the length of $e_x$ is $k$, then we say that $e_x$ is a {\em
    $k$-explanation} (for $h(x)$).
\end{definition}

Definition~\ref{dfn:explanation} defines explanations as
substrings. Intuitively, a correct explanation for an instance is the
``core subset of features'' which is ``decisive''. In the sense that
regardless what other features might be, these ``core features'' alone
determine the outcome of the classification. We focus only on
explaining ``positive'' instances as ``core features'' can be
asymmetrical. An instance is ``negative'' not because certain features
are presented in this instance, but rather the lack of ``core
features''. In this sense, all features in a negative case
``collectively explain'' the negative classification.

We illustrate our notion of explanation as follows.

\begin{example}
  Let $S \!=\! \{00100, 00001, 10000, 11111, 00111,$ $10011\}$, the
  labelling function $h$ is \suchthat{}

  \begin{center}
    $h(11111)=h(00111)=h(10011)=\pos$,\\
    $h(00100)=h(00001)=h(10000)=\npos$.
  \end{center}

  There are $C(5,3) = 10$ 3-explanations for each of the
  three \pos{} strings. Correct 3-explanations include $111\cdot\cdot$
  for $g(11111)$, $\cdot\cdot111$ for $g(00111)$ and $1\cdot\cdot11$
  for $g(10011)$. The substring $00\cdot\cdot1$ is not a correct
  3-explanation for $g(00111)$; neither is $100\cdot\cdot$ correct for
  $g(10011)$.
\end{example}

If a dataset is noise free, then every $\pos$ sample in the dataset
has an explanation, formally:
\begin{proposition}
\label{prop:existance}

Given a dataset $S$, if there is no two strings $s_1, s_2 \in
S$ \suchthat{} $s_1=s_2$ and $h(s_1) \neq h(s_2)$, then for each
$s \in S$, \suchthat{} $h(s) = \pos$, $s$ has a correct non-empty
explanation $e_s$.
\end{proposition}
\begin{proof}
(Sketch.) Since each string $s$ is labelled in only one way, $e_s = s$
  is a correct explanation.
\end{proof}


With correct explanation defined, to measure the correctness of an XAI
algorithm, we define $k$-accuracy as follows.
\begin{definition}
\label{dfn:kAcc}
Given a dataset $S$ over alphabet $\Sigma$ and a classifier $g$, for
each $x \in S$ \suchthat{} $g(x) = h(x) = \pos$, let $e_x$ be a correct
explanation for $h(x)$. Then, the {\em $k$-accuracy} of a
$k$-explanation $e'_x$ is $|\{a \in \Sigma | a$ is in both $e_x$ and
$e'_x\}|/k$.
\end{definition}
Note that $k$-accuracy is defined for correct \pos{}
classifications. As we use $k$-accuracy to measure explainer
performance, we calculate such accuracy only when the prediction
matches with the classification ground truth, i.e., when $g(x)=h(x)$.
Intuitively, we ask for explanations only when the prediction is
correct. This helps us to separate explainer performance from
classifier performance and prevent us from processing meaningless
explanations for wrong predictions. 
%

\section{Constructing Datasets with Explanation}
\label{sec:dataset}
To construct datasets with explanations, we create a grammar $G$ with language $L$. We let $L' \subseteq L$ be
the dataset. A string $s \in L'$ is labelled as $\pos$ \ifaf{} there is
some production rule $r$ in $G$ used more than $t>1$ times. Terminal
symbols associated with the production rule triggering the $\pos$
classification form the explanations for $s$.

We present our dataset and explanation construction with the following
example.

\begin{example}
\label{exp:main}

Given a grammar $G = \langle \{S,B,N,T,Y\},$ $\{0,1\}, P, S \rangle$
with $P$ being the following production rules:

\begin{center}
  \begin{tabular}{lccl}
    $S \rightarrow BB \mid NN \mid \epsilon$ &&&\\
    $B \rightarrow TT \mid YY \mid \epsilon$ &&&
    $N \rightarrow TY \mid YT \mid \epsilon$ \\
    $T \rightarrow 11 \mid 00$ &&&
    $Y \rightarrow 01 \mid 10$ \\
  \end{tabular}
\end{center}

\noindent
From $G$, we create a dataset $L$ containing 8-bit strings such as
11000000 and 00110000.
We let the threshold $t=2$ and use parse trees to count the times of
each production rule is used. For instance, the parse tree for the
string 11000000 is shown in Figure~\ref{fig:parseTree}
(left). Production rules used to generate the string are follows.

\begin{center}
\begin{tabular}{|cc|cc|}
\hline
Production Rule & Uses & Production Rule & Uses \\
\hline
$S \rightarrow BB$ & 1 & $B \rightarrow TT$ & 2 \\
$T \rightarrow 11$ & 1 & $T \rightarrow 00$ & 3 \\
\hline
\end{tabular}
\end{center}

Since the production rule $T \rightarrow 00$ is used 3 times, $3 >
t=2$, 11000000 is \pos{}. The explanation
is $\cdot\cdot000000$, which are terminal symbols in the production
rule $T \rightarrow 00$.

\vspace{-10pt}
\begin{figure}[H]
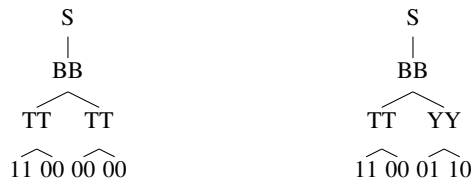

  \Tree[.S [.BB [.TT \text{11} \text{00} ]
      [.TT \text{00} \text{00} ]]]
  \Tree[.S [.BB [.TT \text{11} \text{00} ]
                           [.YY \text{01} \text{10} ]]]
\vspace{-5pt}
  \caption{Parse trees for 11000000 (left) and 11000110
  (right).\label{fig:parseTree}}
\vspace{-10pt}
\end{figure}

On the other hand, for the string 11000110 with its parse tree shown
in Figure~\ref{fig:parseTree} (right), as shown below, there is no
rule triggered more than once. Thus 11000110 is labelled \npos{} and
there is no explanation for 11000110.

\begin{center}
\begin{tabular}{|cc|cc|}
\hline
Production Rule  & Uses & Production Rule  & Uses \\
\hline
$S \rightarrow BB$ & 1 && \\
$B \rightarrow YY$ & 1 & $B \rightarrow TT$ & 1 \\
$T \rightarrow 11$ & 1 & $T \rightarrow 00$ & 1 \\
$Y \rightarrow 01$ & 1 & $Y \rightarrow 10$ & 1 \\
\hline
\end{tabular}
\end{center}
\end{example}

To formalize our explanation dataset approach, we start by defining
{\em explanation-grammar}, the class of grammar we used to generate
datasets, as follows.

\begin{definition}
  \label{def:eGrammar}

  A grammar $\langle N, T, P, S \rangle$ is an {\em
    explanation-grammar (e-grammar)} \ifaf{} all of the following
  conditions hold:

  \begin{enumerate}
  \item
    $N = N_v \cup N_t$, $N_v \cap N_t = \{\}$,

  \item
    For each production rule $r$ in $P$, $r$ is of the form:

    \begin{itemize}
    \item
    $A \rightarrow BC \mid \epsilon$, for $B,C \in N$, if $A \in N_v$,

    \item
    $A \rightarrow abc\ldots$, for $a,b,c \in T$, if $A \in N_t$.
    \end{itemize}
  \end{enumerate}
  For a production rule $r = A \rightarrow \ldots$, we say $r$ is a
  {\em non-terminal rule} if $A \in N_v$; otherwise, $r$ is a {\em
  terminal rule}.

\end{definition}

Grammar $G$ in Example~\ref{exp:main} is an e-grammar. Comparing with
the standard grammar definition (see Section~\ref{sec:exp_classification}), e-grammar
enforces new conditions as follows.
\begin{enumerate}
\item
Each nonterminal symbol can only be the left-hand side
of rules with either terminal or nonterminal symbols on the
right-hand side, but not both.

\item
Production rules with nonterminal symbols on the right-hand side are
called nonterminal rules; otherwise, they are terminal rules.

\item
There are either two nonterminal symbols or a single $\epsilon$ as the
right-hand side of a nonterminal rule.

\item
There are any positive number of terminal symbols as the right-hand
side of a terminal rule.
\end{enumerate}

The following holds trivially from Definition~\ref{def:eGrammar}.

\begin{proposition}
\label{prop:produceAll}

For any dataset $D$ consisting strings, there exists an e-grammar $G$
\suchthat{} the language of $G$ is $D$.
\end{proposition}
\begin{proof}
Let $G = \langle N, T, P, S\rangle$ with $N = \{S\}$, $T$ be all
symbols in $D$, and $P = \{S \rightarrow s | s \in D\}$. It is easy to
see that the language of $G$ is $D$.
\end{proof}

Explanation datasets are created from e-grammars using
Algorithm~\ref{alg:label}, as follows.
\begin{algorithm}
  \caption{Generating explanation datasets containing strings with
    labels and explanations.  \label{alg:label} \vspace{-15pt}}

  \vspace{-10pt}
  {\bf Input:} e-grammar $G$, string length $l$, \pos-threshold $t$

  {\bf Output:} strings with classification and explanation labels

\begin{algorithmic}[1] 
  \STATE Let $D$ be empty

  \WHILE{stop condition not met} \label{line:stop}
    \STATE Randomly generate a string $s$ of length $l$ from
    $G$ \label{line:gen_string}

%
%

    \STATE Let $Rs$ be terminal rules used in generating
    $s$ \\ \suchthat{} each $r \in Rs$ is used more than $t$ times

    \IF {$Rs$ is not empty}
      \STATE Let $e$ be the substring of $s$ formed by terminal
      symbols in $Rs$

      \STATE add $(\pos, s, e)$ to $D$
    \ELSE
      \STATE add $(\npos, s, \{\})$ to $D$
    \ENDIF
    \ENDWHILE
  \STATE \textbf{return} $D$
\end{algorithmic}
\vspace{5pt}
\end{algorithm}
Algorithm~\ref{alg:label} takes an e-grammar $G$, a string
length $l$, and a threshold $t$ as its inputs to produce
a dataset $D$ containing strings with classification label and
explanation substrings. If a string $s$ in $D$ is labelled as $\pos$,
then there is a non-empty explanation $e$ produced for $s$. The stop
condition in Line~\ref{line:stop} determines when to exit from the
{\bf while} loop. It is a combination of: (1) whether sufficiently many
strings have been added to $D$; (2) whether $D$ contains balanced
\pos{} and \npos{} samples; and (3) whether the loop has been running
for too long. To generate a random string from a an e-grammar $G$ in
Line~\ref{line:gen_string}, we repeatedly perform random derivations
until a string with $l$ terminal symbols is produced while
prioritising derivations of terminal rules. If there is no more
terminal rule can be applied to the string when it reaches $l$
terminal symbols, then this string is returned as all non-terminal
symbols can be expanded to $\epsilon$; otherwise, drop this derivation
and start again.

Proposition~\ref{prop:exp-unique} sanctions that
Algorithm~\ref{alg:label} computes unique explanations from
unambiguous grammars.
\begin{proposition}
\label{prop:exp-unique}
Let $e$ be an explanation for some \pos{} string $s$ in a dataset $D$
generated with Algorithm~\ref{alg:label} using grammar $G$ and
$\pos$-threshold $t$. If $G$ is unambiguous, then $e$ is an
explanation for $s$ in any dataset generated from $G$ with $t$.
\end{proposition}
\begin{proof}
(Sketch.) Since $G$ is unambiguous, $s$ has a unique parse tree. To
construct $s$, it always takes more than $t$ invocations of a certain
production rule to make $e$ in any dataset.
\end{proof}

Theorem~\ref{thm:grammar-exp} below is a key result of this work. It
sanctions that Algorithm~\ref{alg:label} generates strings with
correct explanations.

\begin{theorem}
\label{thm:grammar-exp}

Given $D$ generated from Algorithm~\ref{alg:label} with some e-grammar
$G$ and threshold $t$. If for all terminal rules $r$ in $G$, the
right-hand side of $r$ is unique and has the same length, then for
each $(\pos, s, e) \in D$, $e$ is a correct explanation for $s$ being
$\pos$.
\end{theorem}
\begin{proof}
  (Sketch.) To show $e$ is a correct explanation, we need to show for
  all string $s'$ in $D$, if $e$ is a substring of $s'$, $s'$ is
  $\pos$. For each string $s$ in $D$, $s$ can be viewed as a sequence
  of ``composition blocks'' in which each block is the right-hand side
  of some terminal rule in $G$. Since $s$ is $\pos$, there must
  exist a production rule $r^*$ used $t'$ times with $t' > t$. And $e$
  contains $t'$ copies of the right-hand side of $r^*$. Since this is
  the only way of generating $e$ in $D$, for each $s'$ containing $e$,
  $r^*$ must be used $t'$ time as well, which makes $s'$ $\pos$.
\end{proof}

\section{Explanation and G-Complexity}
\label{sec:complexity}


Classification and explanation performances are affected by the
complexity of the grammar used to construct the dataset. We define
{\em grammar complexity} ({\em G-complexity}) as follows.

\begin{definition}
\label{dfn:G-complexity}

Given an e-grammar $G = \langle N, T, P,$ $S \rangle$, $m = |\{r \in
P|$ the right-hand side of $r$ is not $\epsilon \}|$ is the {\em
  G-complexity} of $G$.
Let $D$ be a dataset constructed from $G$ using
Algorithm~\ref{alg:label}, we say that $m$ is a G-complexity of $D$.

\end{definition}

Definition~\ref{dfn:G-complexity} uses the number of production rules
that are not expanded to $\epsilon$
to describe complexities of the grammar and any dataset constructed
from it. This is in the same spirit as the Kolmogorov
Complexity. Kolmogorov Complexity, defined over strings,
is the length of the shortest program that generates the string.
In our context, as datasets are sets of strings, we draw the analogy
between string generating programs and production rules and use the
number of rules as a proxy to measure the dataset complexity.

\begin{example}
\label{exp:G-complexity}

The G-complexity of grammar $G$ given in Example~\ref{exp:main} is
10. 10 is a G-complexity of any datasets generated from $G$ using
Algorithm~\ref{alg:label}.
\end{example}

From Definition~\ref{dfn:G-complexity}, the following holds.
\begin{proposition}
\label{prop:uniqueG}

Every e-grammar has an unique G-complexity. A dataset can have more
than one G-complexity.
\end{proposition}
\begin{proof}
(Sketch.) As G-complexity is defined on e-grammars, one
  obtains it by counting production rules in the grammar.
  As a dataset can be constructed from more than one e-grammar using
  Algorithm~\ref{alg:label}, these e-grammars can have different
  g-complexities, so the dataset can have more than one
  G-complexity. For instance, let $\Delta = \{s_1, \ldots, s_n\}$ be
  the set of strings in a dataset \suchthat{} the length of strings is
  2, and $s_i \neq s_j$ for $i \neq j$, then both e-grammars
  $G_1$ and $G_2$ construct $\Delta$ as follows.

  $G_1 = \langle N, T, P_1, S\rangle$, $G_2 = \langle N, T, P_2,
  S\rangle$, where
  \begin{itemize}
  \item
    $N = \{A_1, \ldots, A_n, B_1, \ldots, B_n\}$,

  \item
    $T$ is the set of symbols in $\Delta$,

  \item
    $P_1$ is \suchthat{} for each $s_i = ab \in \Delta$,

    \begin{tabular}{lcl}
        1. $S \rightarrow A_iB_i \mid \epsilon \in P_1$, &&
        2. $A_i \rightarrow \epsilon \in P_1$,\\
        3. $B_i \rightarrow ab \in P_1$, &&
        4. Nothing else is in $P_1$.\\
    \end{tabular}

  \item
    $P_2$ is \suchthat{} for each $s_i = ab \in \Delta$,

    \begin{tabular}{lcl}
    1. $S \rightarrow A_iB_i \mid \epsilon \in P_2$, &&
    2. $A_i \rightarrow a \in P_2$,\\
    3. $B_i \rightarrow b \in P_2$, &&
    4. Nothing else is in $P_2$.\\
    \end{tabular}

  \item
    $S$ is the start symbol.
  \end{itemize}

\noindent
  G-complexities of $G_1$, $G_2$ are $2n$ and $3n$, respectively.
\end{proof}

\section{Evaluation over SHAP and LIME}
\label{sec:evaluation}
We evaluate the performance of two popular feature importance XAI methods, SHAP and LIME, over generated benchmark datasets using the proposed approach.

{\bf SHapley Additive exPlanations (SHAP)} is a method that gives
individual, thus ``local'', explanations to black-box machine learning
predictions \cite{Lundberg17}. It is based
on the coalitional game theory concept {\em Shapley value}.
Shapley value is defined to answers the question: ``What is the
fairest way for a coalition to divide its payout among the players''?
It assumes that payouts should be assigned to players in a game
depending on their contribution towards total payout. In a machine
learning context, feature values are ``player''; and the prediction is
the ``total payout''. The Shapley value of a feature represents its
contribution to the prediction and thus explains the prediction. SHAP
is ``model-agnostic'' thus independent of underlying prediction
models. For a data point $x$, SHAP computes the marginal contribution
of each feature to the prediction of $x$. In this work, we use the
tree-based model, TreeSHAP, for estimating Shapley values of features
introduced in \cite{Lundberg18}, as which is shown to be a superior
method than the Kernel SHAP introduced
in \cite{Lundberg17}.


{\bf Local Interpretable Model-Agnostic Explanations (LIME)} is
another method to explain individual predictions of
machine learning models. LIME also is a model-agnostic
approach, so it is applicable to any classifier \cite{Ribeiro0G16}.
LIME tests how predictions change when a user perturbs the input
data. Given a black box model $f$ and a data instance $x$, to
explain the prediction of $x$ made with $f$, LIME generates a set of
perturbed instances around $x$ and compute and their corresponding
predictions. It then creates an interpretable model $g$ based on
generated data to approximate and explain $f$. LIME provides an
explanation as a list of feature contributions to the prediction of
the instance $x$. This highlights feature changes that have the most
influence to the prediction.

To evaluate the performance of SHAP and LIME, we create 5 datasets containing 1000, \ldots,
5000 strings from each e-grammar (see Definition \ref{def:eGrammar} in Section \ref{sec:dataset}), respectively. Each e-grammar is defined on 8 nonterminal
symbols, contains 40 rules and 2 terminal symbols. The
maximum length of the right-hand side of terminal rules is
2. The \pos-threshold $t$ is 8. We compute $k$-accuracy for
$k=8$. We use a Random Forest classifier with 100 trees, split
training and testing with a 80/20 ratio, and returns the average AUC
and $k$-accuracy for SHAP and LIME.
Results are shown in Figure~\ref{fig:different_sample_size}. We see
that SHAP performs consistently better than LIME (26\% higher on
average). Classification and SHAP performances improve as the number
of samples increases whereas LIME performance largely remains.

\begin{figure}[h]
\centerline{
        \includegraphics[trim=20 10 20 5,
        width=0.45\textwidth]{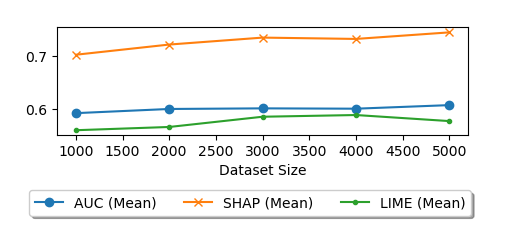}
        \includegraphics[trim=20 10 20 5,
        width=0.45\textwidth]{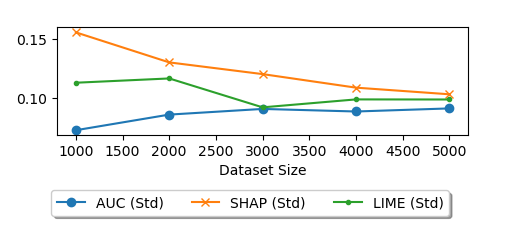}
}
\vspace{-4pt}
\caption{Classification AUC and explanation $k$-accuracy under
        different dataset sizes (Left: mean; Right: standard
        deviation). \label{fig:different_sample_size}}
\vspace{-20pt}
\end{figure}

To illustrate relations amongst G-complexity, classification and
explanation, we create datasets with specified g-complexities by
randomly generating e-grammars with the desired number of production
rules and evaluate classification and explanation performances. The
results are shown in Figure~\ref{fig:complexity}. For
string lengths 20 to 35 and G-Complexities 20 to 60, we construct 100
e-grammars each generating a dataset containing 10,000 strings with
the ratio between \pos{} and \npos{} samples in the range of $[0.4,
0.6]$. The \pos-threshold $t$ and $k$ in $k$-accuracy are 6 for
string length 20, and 8 for the rest.
\vspace{-20pt}
\begin{figure*}
        \includegraphics[trim=118 35 118 0,width=1\textwidth]{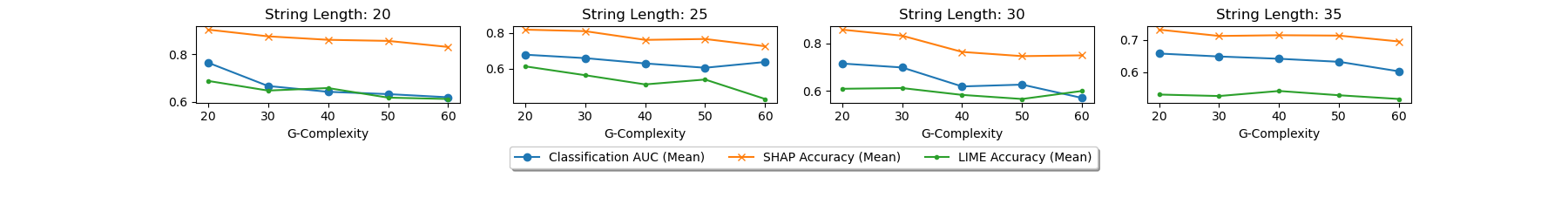}
        \includegraphics[trim=118 40 118 0,width=1\textwidth]{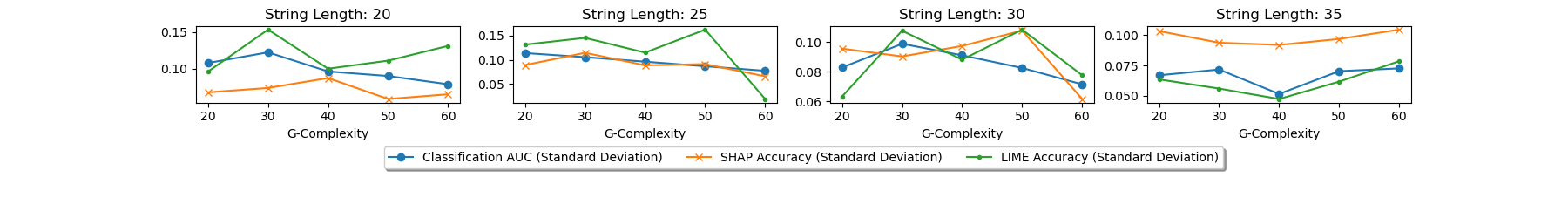}
\vspace{-15pt}
\caption{Classification AUC and explanation $k$-accuracy for different
        g-complexities (Top: mean; Bottom: standard
        deviation).\label{fig:complexity}}
\vspace{-30pt}
\end{figure*}

From Figure~\ref{fig:complexity}, we see that for all string lengths,
classification accuracy, SHAP and LIME performances all negatively
correlate to G-Complexity, as summarised in Table~\ref{table:corr},
with SHAP performing consistently better than LIME in all cases (37\%
higher on average). SHAP and LIME accuracies are
positively correlated to classification AUC at 0.62 and 0.46,
respectively, averaging for all string lengths in
Figure~\ref{fig:complexity}.

\vspace{-30pt}
\begin{table}[h]
\begin{small}
\begin{center}
\caption{Correlations between G-Complexity and mean performances of
classification AUC, SHAP and LIME $k$-accuracies.
\label{table:corr}}
\begin{tabular}{|c|ccc|}
\hline
String Length & Classification & SHAP $k$-acc. & LIME $k$-acc. \\
\hline
20 & -0.88 & -0.97 & -0.92 \\
25 & -0.76 & -0.95 & -0.91 \\
30 & -0.95 & -0.93 & -0.53 \\
35 & -0.94 & -0.88 & -0.45 \\
\hline
\end{tabular}
\end{center}
\vspace{-10pt}
\end{small}
\end{table}

To further validate these results, we expand the size of alphabet to
4, set string length to 25 and repeat the experiment. Results are
in Figure~\ref{fig:alphabet4}. Classification and explanation
performances are negatively correlate to G-Complexity with SHAP
performing better than LIME (49\% higher on average).

\begin{figure}[H]
\centerline{
        \includegraphics[trim=20 10 20 0,
        width=0.45\textwidth]{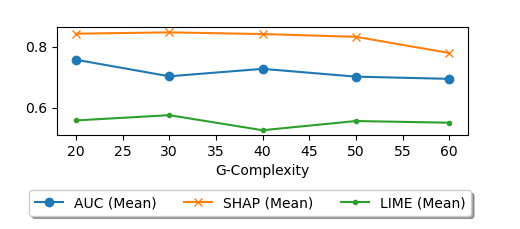}
        \includegraphics[trim=20 10 20 0,
        width=0.45\textwidth]{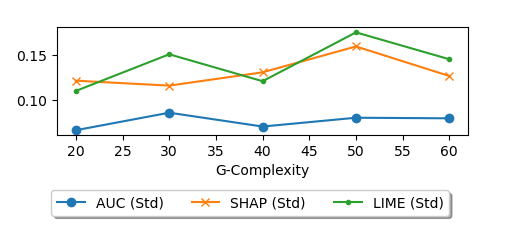}
}
\caption{Classification AUC and explanation $k$-accuracy for alphabet
        size 4 (Left: mean; Right: standard
        deviation).\label{fig:alphabet4}}
\end{figure}

\section{Related Work}
\label{sec:relatedWork}


A survey on machine learning
interpretability is presented in \cite{Carvalho19}. It provides an overview on interpretability while
focusing on the societal impact and interpretability metrics. It
presents 12 model-agnostic explanation methods including SHAP and LIME
as well as Anchors \cite{Ribeiro18} and Influence
Functions \cite{Koh17}. There is no quantitative comparison of these
methods in the study.

A comparison of
LIME, DeepLIFT \cite{Shrikumar17} and SHAP with human
explanations is conducted in \cite{Lundberg17}. They report a stronger agreement between human
explanations and SHAP than with LIME. \cite{Lundberg19} show
that SHAP is more consistent with human intuition in their experiments
than other approaches. \cite{Honegger18} present a comparison of SHAP and LIME using their
Explanation Consistency Framework. Three requirements are
proposed: Identity - identical objects
must have identical explanations; Separability - nonidentical objects
cannot have identical explanations; and Stability - similar objects
must have similar explanations. They show that SHAP meets all
requirements LIME fails at identity. However, they do not measure explanation
accuracy. \cite{Robnik-SikonjaB18} define properties for outputs
generated by explanation methods including accuracy, fidelity,
consistency, stability, comprehensibility, importance, novelty,
and others. However, they do not provide any concrete approach for
measuring these properties. While presenting an algorithm for generating
counterfactual explanation, \cite{Rathi19} show that LIME does not guarantee to
perfectly distribute prediction amongst feature values whereas SHAP
does. Thus LIME does not offer a globally consistent explanation as
SHAP. While pointing out this theoretical difference, no
quantitative evaluation is performed. \cite{Wilson18} present the concept of ``Three Cs of
Interpretability'', completeness, correctness and
compactness. Completeness refers to the coverage of the explanation in
terms of the number of instances comprised by the
explanation. Correctness means the explanation must be
true. Compactness means the explanation should be succinct. They
present their study in a healthcare setting with similar examples to
the query instance considered as explanations. They do not study
explanations in terms of features.

A framework to hide the biases of black-box
classifier is proposed in \cite{Slack20}. Specifically, they use biases to make black-box classifier
discriminatory to effectively fool explanation techniques such as SHAP
and LIME into generating incorrect explanations which do not reflect
the discriminatory biases in the data. They find that LIME is more
vulnerable than SHAP to their attacks.



\section{Conclusion}
\label{sec:conclusion}

A key challenge in current XAI research is to develop robust ways to
evaluate explanation methods. The lack of qualitative evaluation is
largely due to the missing of explanation ground truth in the existing
literature. In this work, while focusing on binary classification,
we present a definition for correct explanation,
a metric for explanation evaluation and provide an algorithm for
constructing datasets with correct explanation ground truth for
quantitatively evaluating model agnostic explanation algorithms. We
create datasets as languages of grammars and set explanations as
substrings created from repeated application of production rules. We
introduce G-complexity, modelled after the Kolmogorov Complexity of
strings, to describe dataset complexity and show that both
classification and explanation become harder as datasets become more
complex. We evaluate SHAP and LIME with our approach and show that
SHAP perform better than LIME throughout and SHAP has a stronger
correlation to classification performance than LIME.
For future work, we will perform human user studies to see whether our
notion of correct explanation corresponds to human explanations. We
will experiment with other model-agnostic explainer such as
Anchors \cite{Ribeiro18}. As our notion of correct explanation is
similar to theirs in spirit, it will be interesting to see how Anchors
performs against SHAP. Lastly, we will also extend our approach to
multi-class classification and regression.

%
%
 \bibliographystyle{splncs04}
 \bibliography{mybibliography}
%
%
%
%
%
\end{document}